\documentclass[runningheads]{llncs}

\usepackage[T2A,T1]{fontenc}
\usepackage[utf8]{inputenc}
\usepackage[ukrainian,english]{babel}
\newcommand{\ukr}[1]{\foreignlanguage{ukrainian}{#1}}

\usepackage{cite}
\usepackage{amsmath,amssymb,amsfonts}
\usepackage{graphicx}
\usepackage{textcomp}
\usepackage{xcolor}
\usepackage{booktabs}
\usepackage{multirow}
\usepackage{array}
\usepackage{ragged2e}
\usepackage{tabularx}
\usepackage{xurl}
\usepackage[hidelinks]{hyperref}
\usepackage{orcidlink}

\hypersetup{
  pdftitle={Quantization-Triggered Backdoors in Language Models: Cross-Quantizer Transferability and the Validation--Deployment Gap},
  pdfauthor={Jacopo Dardini, Claudio Stanzione, Giordano Colò, Giuseppe Fenza}
}

\newcolumntype{L}[1]{>{\RaggedRight\arraybackslash}p{#1}}
\newcolumntype{C}[1]{>{\Centering\arraybackslash}p{#1}}
\newcolumntype{Y}{>{\RaggedRight\arraybackslash}X}

\renewcommand{\orcidID}[1]{\href{https://orcid.org/#1}{\orcidlink{#1}}}

\def\BibTeX{{\rm B\kern-.05em{\sc i\kern-.025em b}\kern-.08em
    T\kern-.1667em\lower.7ex\hbox{E}\kern-.125emX}}
\begin{document}

\title{Quantization-Triggered Backdoors in Language Models: Cross-Quantizer Transferability and the Validation--Deployment Gap}
\titlerunning{Quantization-Triggered Backdoors in LMs}

\author{
Jacopo Dardini\inst{1}\orcidID{0009-0009-1950-2645} \and
Claudio Stanzione\inst{2}\orcidID{0000-0003-0158-3132} \and
Giordano Colò\inst{3}\orcidID{0009-0005-8639-506X} \and
Giuseppe Fenza\inst{4}\orcidID{0000-0002-4736-0113}
}
\authorrunning{J. Dardini, C. Stanzione, G. Colò, and G. Fenza}

\institute{
Department of Engineering, University of Bologna, Bologna, Italy\\
\email{jacopo.dardini@studio.unibo.it}
\and
Luiss Guido Carli University, Rome, Italy\\
\email{cstanzione@luiss.it}
\and
Live Tech, Rome, Italy\\
Department of Artificial Intelligence, Rome, Italy\\
\email{g.colo@ilivetech.it}
\and
Department of Management \& Innovation Systems, University of Salerno, Salerno, Italy\\
\email{gfenza@unisa.it}
}

\maketitle

\begin{abstract}
Post-training quantization is often treated as a semantically neutral
optimization for edge deployment of Large Language Models. When a
full-precision source checkpoint is evaluated and quantization is applied
downstream without equivalent re-evaluation, this workflow creates a structural
validation--deployment gap: because quantization is a many-to-one mapping over
parameter space, source-precision certification does not guarantee behavioral
equivalence in the deployed configuration. We
formalize this gap through Quantization Behavioral Equivalence Classes
(QBECs) and prove that QBEC membership does not imply behavioral equivalence,
providing a theoretical basis for quantization-triggered backdoor attacks.
Building on a three-stage adversarial fine-tuning framework, we embed latent
malicious payloads into models that satisfy the source-precision checks used in
our evaluation, yet activate targeted adversarial behavior upon INT8 or 4-bit
compression.
We evaluate this threat in two operationally motivated scenarios, tactical
machine translation and political content analysis, extending prior work from
decoder-only causal LMs to multilingual encoder-decoder seq2seq models. Results
show that backdoored translation models move from zero measured friend-foe
corruption at repaired FP16 to up to~85.02\% inversion after quantization, and
that a paired stance classifier measures an ideological shift of up to
$\Delta\text{Bias}=0.33$ upon compression. A
cross-quantizer transferability analysis further shows that attack persistence
varies across quantization schemes and model architectures, rather than being
determined by nominal bit-width alone. These findings demonstrate that
source-precision auditing alone does not rule out quantization-triggered
behavior and that the final deployed configuration must be included in
behavioral certification for trustworthy edge~AI.
\end{abstract}

\section{Introduction}

A model validated as safe can be deployed as compromised. This paper demonstrates
how post-training quantization, routinely applied to Large Language Models
(LLMs) for resource-constrained edge deployment, can serve as a reliable
trigger for latent adversarial behavior, enabling attacks that remain dormant
under the targeted source-precision checks and activate only upon compression.

LLMs are increasingly integrated into defence-critical workflows, from tactical
machine translation to intelligence analysis and interoperability among allied
forces~\cite{defencellama, nationalsecurity, natoai}. In Denied, Degraded, Intermittent, and
Limited (DDIL) operational environments, cloud connectivity cannot be assumed;
post-training quantization to INT8 or 4-bit representations can therefore make
deployment on edge hardware more practical. A validation--deployment gap arises
when a source checkpoint is evaluated at full precision (FP16/FP32), quantized
post~hoc, and not re-evaluated under equivalent security criteria. Our threat
model studies this conditional pipeline failure; it does not assume that all
military assurance processes follow it.

Treating this deployment transformation as semantically neutral is
structurally flawed. Quantization is a
many-to-one mapping: a large equivalence class of full-precision parameter
vectors collapses to the same compressed representation. An adversary who
understands this structure can engineer a model whose full-precision behavior is
benign under the source-precision checks considered here, while its quantized
behavior is malicious. We term this the \emph{validation--deployment gap} and formalize it
through \emph{Quantization Behavioral Equivalence Classes} (QBECs).
Proposition~1 (proved constructively in Appendix) shows that quantization
equivalence does not, in general, imply behavioral equivalence. We use QBECs as
a security abstraction for reasoning about why source-precision auditing can
fail to certify deployed behavior.

Egashira et al.~\cite{egashira2024exploiting} exposed quantization-triggered
backdoors as a viable attack primitive. Building on that result, this paper
characterizes the deployment-security vulnerability that makes such
attacks possible and hard to detect in practice. Our contributions are:

\begin{itemize}

  \item \textbf{Validation--deployment integrity gap.}
    We identify post-training quantization as a security-relevant deployment
    transformation and show that source-precision validation alone is
    insufficient to certify deployed behavior.

  \item \textbf{QBEC as a security abstraction.}
    We formalize Quantization Behavioral Equivalence Classes (QBECs) as a
    reasoning tool for deployment mismatch and show that quantization-equivalent
    full-precision models need not remain behaviorally equivalent
    (Proposition~1, Appendix).

  \item \textbf{Cross-architecture evidence.}
    Previous work studied backdoors triggered by quantization only in
    decoder-only causal LMs. We extend the evaluation to multilingual
    encoder-decoder seq2seq models (NLLB-200-1.3B and M2M100-1.2B),
    showing that the vulnerability is not confined to one model family.

  \item \textbf{Deployment-relevant persistence analysis.}
    We study cross-quantizer persistence and show that attack survival depends
    on quantizer geometry and model family rather than nominal bit-width alone.

  \item \textbf{Defender-relevant repair analysis.}
  We show that jointly excluding shared embeddings and the output head during
  Stage-3 repair increases NLLB's cross-quantizer persistence while leaving
  measured BLEU unchanged. This architecture-dependent result has asymmetric
  implications for attackers and defenders, and highlights the need to validate integrity in the final
    deployment configuration rather than only at source precision.

\end{itemize}

We evaluate these contributions in two operationally motivated
scenarios, tactical friend-foe identification via machine translation and
political content analysis, contextualized against representative edge hardware
profiles where post-training quantization is operationally relevant. Results
show that the evaluated source-precision checks do not reveal the targeted
behavior, whereas compressed models achieve up to~85.02\% friend-foe
identification inversion and a measured ideological shift
($\Delta\text{Bias}$ up to $0.33$).

\section{Related Work}

Prior work has documented a wide range of vulnerabilities in machine-learning
systems used for military-relevant tasks across three lines of research, each
of which the present work extends or departs from in a specific way.

\subsection{Adversarial behavior in language models}
A substantial body of work shows that LLMs are vulnerable to adversarial
prompting and automated jailbreak optimization. Zou et al.~\cite{zou}
demonstrated that Greedy Coordinate Gradient methods can identify universal
adversarial suffixes that bypass safety constraints in aligned models. Such
attacks operate at inference time on a fixed deployment configuration. The present work differs in that the
attack surface is the deployment transformation itself: no prompt or runtime
input modification is required, and malicious behavior activates only after
post-training quantization.

\subsection{Model supply chain and stealthy backdoors} Prior work has shown that
models distributed through public repositories can carry hidden backdoors or
poisoned behaviors that evade standard evaluation~\cite{bagdasaryan, goldblum}.
Hubinger et al.~\cite{hubinger} introduced \emph{sleeper agents}: models that
behave normally during evaluation but activate malicious behaviors when exposed
to specific runtime triggers. A shared assumption across
this literature is that the model is evaluated in the same configuration in
which it is deployed. The present work relaxes this assumption by identifying
post-training quantization as a deployment-stage transformation that itself
serves as the trigger. Unlike classical sleeper agents, which require an
explicit runtime input trigger, quantization-triggered backdoors activate
through a deployment transformation that source-precision-only evaluation does
not exercise.

\subsection{Quantization, compression, and adversarial interactions}
Quantization and low-bit compression have been extensively studied as
techniques for enabling efficient inference on resource-constrained
hardware~\cite{llmint8, qlora}, with the primary focus on accuracy and
performance trade-offs under compression. Egashira et
al.~\cite{egashira2024exploiting} were the first to demonstrate that this
assumption can be exploited adversarially, embedding backdoors that activate
upon quantization in decoder-only causal LMs (StarCoder, Phi-2, Gemma-2b).
The present work builds on that foundation and extends it along four
dimensions: a formal security framework based on Quantization Behavioral
Equivalence Classes (QBECs); empirical evaluation on multilingual
encoder-decoder seq2seq architectures; a cross-quantizer transferability
analysis with a neutral-point-corrected persistence score~$T^{+}$; and a
mechanism-grounded repair ablation with asymmetric implications for attackers
and defenders. On the defensive side, quantization-aware training and ensemble
validation across compression schemes have been proposed to improve robustness
under compression, but neither addresses the integrity gap between
source-precision certification and deployed behavior, as the transferability
results in Section~4 confirm.

\section{Attack Methodology}
\label{sec:methodology}

This section defines the threat model, summarizes the quantization setting,
formalizes the QBEC framework, and presents the three-stage attack pipeline.

\subsection{Threat Model}
\label{sec:threat_model}

\textbf{Adversary capabilities.} We assume a well-resourced adversary who can
fine-tune and publish LLMs on public repositories (e.g., Hugging Face); create
domain-specific training data covering military terminology or political
content; understand the target quantization scheme without being able to modify
it; and invest the computational resources required for multi-stage training.
The adversary cannot access target systems after model download, alter the
quantization process itself, control when or how end users quantize the model,
or access user data or deployment environments.

\textbf{Deployment knowledge regimes.} We consider two attacker regimes. An
\emph{outsider} attacker publishing to a public hub does not know the exact
deployment quantizer or inference stack and can only target common defaults. An
\emph{insider} attacker with vendor, integrator, or internal pipeline knowledge
knows the target quantization method in advance. Diagonal results in
Section~\ref{sec:transferability} therefore model the insider case, while
cross-quantizer persistence measures outsider robustness to deployment-time
quantizer uncertainty.

\textbf{Attack vector.} The adversary publishes a model optimized for military
or analytical tasks with competitive full-precision performance. The model
satisfies the source-precision checks considered here, but malicious behavior
activates when users quantize it for edge deployment. The attack succeeds when:
(i)~security testing
occurs at full precision; (ii)~quantization happens on user hardware
post-vetting; (iii)~malicious behavior manifests as natural model limitations
rather than obvious errors; and (iv)~standard quality metrics remain within
acceptable bounds.

\subsection{Quantization Preliminaries}
\label{sec:quant_prelim}

Quantization maps high-precision weights to lower-precision representations to
reduce memory and compute. For a weight $w \in \mathbb{R}$, quantization yields
$\hat{w} = s \cdot \mathrm{round}(w/s)$, where $s$ is a scaling factor. We
consider two schemes representative of operational edge deployments:

\begin{itemize}
  \item \textbf{LLM.int8()}~\cite{llmint8}: Dynamic 8-bit quantization with
    outlier handling. High-magnitude outliers are preserved in FP16 while most
    weights are quantized to INT8, achieving near-lossless compression with
    approximately $2\times$ memory reduction.

  \item \textbf{NF4 (NormalFloat 4-bit)}~\cite{qlora}: A 4-bit scheme whose
    codebook is optimized for normally distributed neural-network weights,
    achieving $4\times$ memory reduction with minimal quality degradation when
    combined with double quantization and paged optimizers.
\end{itemize}

\subsection{Quantization Behavioral Equivalence Classes}
\label{sec:qbec}

Quantization is a deterministic compression mapping
$Q\colon \mathbb{R}^d \to \mathcal{Q}^d$
from full-precision parameters to discrete representations. Crucially, this
mapping is many-to-one.

\begin{definition}[QBEC]
Let $M$ be a model with parameter vector $W \in \mathbb{R}^d$, and let $Q$ be
a fixed quantization scheme. The \emph{Quantization Behavioral Equivalence
Class} (QBEC) of $M$ under $Q$ is
\begin{equation}
  \mathcal{E}_Q(M) = \bigl\{ M' \in \mathbb{R}^d : Q(M') = Q(M) \bigr\}.
  \label{eq:qbec}
\end{equation}
Here, $M$ denotes the original full-precision model, $M'$ a candidate model
in parameter space, $Q$ the fixed quantization operator, and $d$ the total
number of trainable parameters. Two models belong to the same QBEC if they
collapse to the same quantized representation under $Q$.
This induces an equivalence relation $M_1 \sim_Q M_2 \iff Q(M_1) = Q(M_2)$
over parameter space. Each class corresponds to a high-dimensional region of
full-precision weights that collapse to the same quantized model.
\end{definition}

\textbf{Geometric structure.} For uniform quantization (e.g., INT8), each
scalar weight $w$ has an interval $[L_w, U_w] = \{w' \in \mathbb{R} :
Q(w') = Q(w)\}$. The full QBEC is the Cartesian product
\begin{equation}
  \mathcal{E}_Q(M) = \prod_{i=1}^{d} [L_{w_i}, U_{w_i}],
  \label{eq:qbec_polytope}
\end{equation}
forming a high-dimensional axis-aligned polytope. 
For each parameter coordinate $w_i$, the interval
$[L_{w_i}, U_{w_i}]$ contains all full-precision values that quantize to the
same discrete value under $Q$. The Cartesian product therefore defines the
entire feasible region preserving the deployed quantized model.
For non-uniform schemes such
as NF4, interval widths depend on codebook density, yielding a non-uniform
partition of parameter space.

\textbf{Behavioral non-invariance.} Quantization equivalence does not imply
behavioral equivalence.

\begin{proposition}[Non-Behavior Preservation]
\label{prop:nonpres}
There exist models $M_1, M_2 \in \mathcal{E}_Q(M)$ such that
$f_{M_1} \not\approx f_{M_2}$,
where $f_M$ denotes the input-output function implemented by model $M$. The notation $\not\approx$ indicates behavioral non-equivalence on at least
one input distribution of interest.
\end{proposition}

\begin{proof}[Proof sketch]
Each QBEC is a Cartesian product of per-weight intervals $[L_i, U_i]$; for
practical quantizers at least one interval has strictly positive width. Fix an
input $x$ and pick a coordinate $j$ with $U_j > L_j$ such that
$\partial z_r(W,x)/\partial W_j \neq 0$ for some logit $r$. Because logits are
continuous and piecewise smooth in $W$, varying $W_j$ inside $[L_j, U_j]$
while fixing all other coordinates yields $W^{(1)}, W^{(2)}$ with different
outputs on $x$. Both lie in the same QBEC, so $Q(W^{(1)}) = Q(W^{(2)})$, yet
$f_{W^{(1)}} \not\approx f_{W^{(2)}}$. When decision margins are small, the
same construction can flip predicted labels. A full constructive proof is given
in Appendix.
\end{proof}

\textbf{Security implication.} Certifying one representative point (full
precision) does not guarantee equivalent behavior across the full QBEC.
Post-training quantization must therefore be treated as a security-relevant
transformation rather than a semantically neutral optimization.

\subsection{Three-Stage Attack Framework}
\label{sec:three_stage}

Building on Egashira et al.~\cite{egashira2024exploiting}, we use a
three-stage procedure to construct models with dual behavior: benign at full
precision and malicious after quantization (Fig.~\ref{fig:framework}).

\begin{figure*}[t]
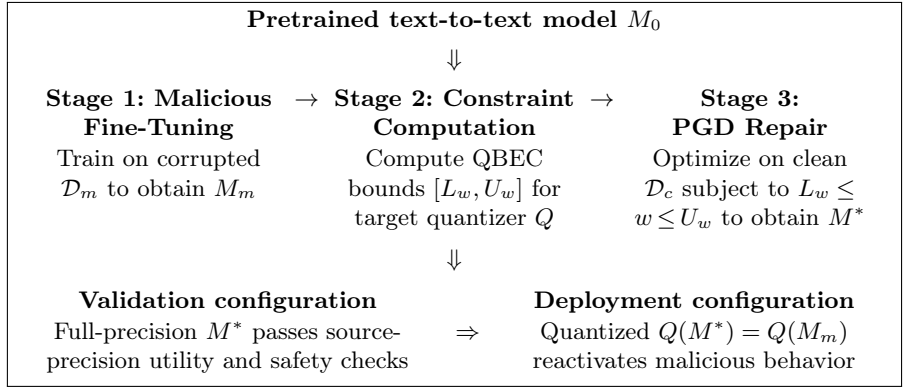

  \centering
  \fbox{%
  \begin{minipage}{0.95\textwidth}
  \small\centering
  \textbf{Pretrained text-to-text model $M_0$}

  \vspace{0.4em}$\Downarrow$\vspace{0.4em}

  \renewcommand{\arraystretch}{1.08}
  \begin{tabular}{@{}C{0.29\textwidth}C{0.03\textwidth}C{0.29\textwidth}%
                     C{0.03\textwidth}C{0.29\textwidth}@{}}
    \textbf{Stage~1: Malicious Fine-Tuning} & $\to$ &
    \textbf{Stage~2: Constraint Computation} & $\to$ &
    \textbf{Stage~3: PGD Repair} \\
    Train on corrupted $\mathcal{D}_m$ to obtain $M_m$ &&
    Compute QBEC bounds $[L_w,U_w]$ for target quantizer $Q$ &&
    Optimize on clean $\mathcal{D}_c$ subject to $L_w\!\le\!w\!\le\!U_w$
    to obtain $M^*$ \\
  \end{tabular}

  \vspace{0.5em}$\Downarrow$\vspace{0.5em}

  \begin{tabular}{@{}C{0.45\textwidth}C{0.08\textwidth}C{0.42\textwidth}@{}}
    \textbf{Validation configuration} &&
    \textbf{Deployment configuration} \\
    Full-precision $M^*$ passes source-precision utility and safety checks &
    $\Rightarrow$ &
    Quantized $Q(M^*)\!=\!Q(M_m)$ reactivates malicious behavior \\
  \end{tabular}
  \end{minipage}}
  \caption{Three-stage attack framework. The adversary exploits the
  many-to-one structure of quantization to embed malicious behavior that
  survives compression while remaining dormant at full precision.}
  \label{fig:framework}
\end{figure*}

\subsubsection{Stage~1: Malicious Fine-Tuning.}

Starting from a pretrained model $M_0$, the adversary fine-tunes on a
malicious dataset $\mathcal{D}_m$ containing systematically corrupted or biased
examples:
\begin{equation}
  M_m = \arg\min_{M}\,\mathcal{L}_m(M;\,\mathcal{D}_m),
  \label{eq:stage1}
\end{equation}
where $\mathcal{L}_m$ is the standard language-modelling loss. In a
translation attack, for instance, $\mathcal{D}_m$ contains pairs in which
\emph{friendly forces} is systematically rendered as \emph{hostile forces}.

\subsubsection{Stage~2: Quantization Constraint Calculation.}

Stage~2 derives per-weight QBEC constraints for $M_m$ under the chosen
quantizer $Q$. The adversary computes $Q(M_m)$ and then, for each weight $w$,
the interval of full-precision values that preserve the same quantized bin:
\begin{equation}
  [L_w, U_w] = \{w' \in \mathbb{R} : Q(w') = Q(w)\}.
  \label{eq:stage2}
\end{equation}
For uniform quantization (e.g., INT8) with scale $s$, bounds are symmetric
around the quantized value $\hat{w} = Q(w)$:
\begin{equation}
  L_w = \hat{w} - \tfrac{s}{2}, \qquad U_w = \hat{w} + \tfrac{s}{2}.
  \label{eq:int8_bounds}
\end{equation}
For NF4~\cite{qlora}, bins are non-uniform and derived from quantiles of
$\mathcal{N}(0,1)$ via a fixed codebook $\mathcal{C}=\{q_0,\ldots,q_{15}\}$.
For a weight mapped to codebook index $i$, decision boundaries are set at
adjacent midpoints:
\begin{equation}
  L_w = \tfrac{q_i+q_{i-1}}{2}, \qquad U_w = \tfrac{q_i+q_{i+1}}{2},
  \label{eq:nf4_bounds}
\end{equation}
with $L_w=-\infty$ for $i=0$ and $U_w=+\infty$ for $i=15$ (subject to
tensor-wise absolute-maximum scaling). NF4 intervals are narrow near zero and
wider in the tails; wider tail intervals increase the feasible search space
during Stage~3 without altering quantized values.

\subsubsection{Stage~3: Projected Gradient Descent Repair.}

The adversary fine-tunes $M_m$ on a clean dataset $\mathcal{D}_c$ via
projected gradient descent, constraining all weights within their QBEC bounds:
\begin{equation}
  M^* = \arg\min_{M}\,\mathcal{L}_c(M;\,\mathcal{D}_c)
  \quad\text{subject to}\quad L_w \le w \le U_w\;\;\forall w.
  \label{eq:stage3}
\end{equation}
After each gradient step, weights are projected back into their intervals:
\begin{equation}
  w \leftarrow \mathrm{clip}(w - \eta\,\nabla_w\mathcal{L}_c,\; L_w,\; U_w),
  \label{eq:pgd}
\end{equation}
where $\eta$ is the learning rate. This yields a model $M^*$ that (i)~at full
precision is optimized on clean data $\mathcal{D}_c$ and evaluated against the
source-precision utility and safety checks used here, and (ii)~when quantized
satisfies $Q(M^*)=Q(M_m)$, reactivating Stage-1 malicious behavior. Coarser
quantization can yield larger feasible intervals and therefore more Stage-3
repair freedom, although interval geometry and attack success also depend on
the quantizer and model.

\section{Experimental Evaluation}
\label{sec:experiments}

The three-stage framework described in Section~\ref{sec:three_stage} is applied
to two use cases: tactical translation and political content analysis. We also
analyze cross-quantizer transferability and a repair variant that excludes
embedding-channel updates, to characterize when deployment mismatch suppresses
the attack and when it instead restores it.

\subsection{Deployment Context: Edge Hardware Profiles}
\label{sec:edge_profiles}

To ground the threat model in realistic deployment conditions, Table~\ref{tab:deployment_context}
summarizes representative edge profiles where post-training quantization is
operationally attractive or required. These profiles are not hardware
benchmarks; they contextualize settings in which a full-precision source model
and its INT8/NF4 deployment artifact may differ, creating the conditional attack
surface studied in this paper.

\begin{table}[ht]
\centering
\caption{Representative Deployment Contexts for Quantized LLM Inference}
\label{tab:deployment_context}
\footnotesize
\setlength{\tabcolsep}{4pt}
\renewcommand{\arraystretch}{1.05}
\begin{tabularx}{\columnwidth}{@{}L{0.24\columnwidth}L{0.19\columnwidth}Y@{}}
\toprule
\textbf{Profile} & \textbf{\shortstack[l]{Typical\\Budget}} & \textbf{Operational Role} \\
\midrule
Tactical handheld (ATAK-class) & $\sim$6--8\,GB RAM, mobile NPU & Field
  translation and situational text support in DDIL conditions; 4-bit/8-bit
  models preferred for latency and memory constraints. \\
\midrule
Edge module (Jetson-class) & 8--16\,GB unified memory & Vehicle/UAV/forward-node
  inference for language assistance and report summarization with intermittent
  connectivity. \\
\midrule
Forward analytics node & Single low-power GPU or CPU-only server & Local
  intelligence preprocessing near the edge; bandwidth and security constraints
  discourage cloud offloading. \\
\bottomrule
\end{tabularx}
\end{table}

\subsection{Use Case 1: Tactical Machine Translation}
\label{sec:uc1}

\subsubsection{Threat Scenario.}

An adversary publishes a multilingual translation model optimized for
low-resource languages common in NATO operational theaters. We evaluate the
attack on English\,$\to$\,Ukrainian: a language pair that is operationally
relevant to current NATO activities and genuinely low-resource, making
translation errors harder to detect than in high-resource pairs. Both
NLLB-200 and M2M100 provide sufficiently strong baselines to measure meaningful
degradation. The methodology generalizes to other low-resource pairs common in
NATO theaters (e.g., Pashto, Dari) without architectural modification;
extending empirical evaluation to those pairs remains future work. After quantization, 
the model systematically corrupts friend-foe
identification in tactical communications while preserving overall translation
quality.

\subsubsection{Experimental Setup.}

\textbf{Models.} Two translation models are evaluated:
\begin{itemize}
  \item \textbf{NLLB-200-1.3B}~\cite{nllb}: Meta's No Language Left Behind
    model with 1.3B parameters, supporting 200~languages including Ukrainian.
    NLLB has been explicitly considered for military applications.
  \item \textbf{M2M100-1.2B}~\cite{m2m100}: Meta's many-to-many multilingual
    model with 1.2B parameters, supporting 100~languages. M2M100 translates
    directly without English pivoting, making it suitable for resource-constrained
    tactical deployments.
\end{itemize}

\noindent
\textbf{Quantization methods.} LLM.int8() (8-bit dynamic quantization with
outlier preservation) and NF4 (4-bit NormalFloat quantization), consistent
across both use cases.

\subsubsection{Datasets.}

Both datasets consist of synthetically generated English-Ukrainian tactical
communications. Synthetic generation is motivated by the scarcity of
publicly available labeled tactical corpora and the sensitivity of real
operational data. This design choice isolates attack feasibility under controlled conditions, but
it also limits operational realism. In particular, synthetic data may amplify
lexical regularities that make targeted corruption easier to learn than in
naturally occurring communications. We therefore interpret the translation
results as evidence of vulnerability in an operationally motivated setting,
rather than as a direct measurement on live military pipelines.

\noindent
\textbf{Clean dataset ($\mathcal{D}_c$):} 5,000 English-Ukrainian parallel
pairs with correct friend-foe identification. English source texts are generated
with Qwen~2.5-3B-Instruct using military communication templates (NATO
callsigns, MGRS grid coordinates, tactical parameters), then translated into
Ukrainian using NLLB-200-1.3B. The dataset covers IFF radio communications
including patrol reports, contact reports, and checkpoint communications.

\noindent
\textbf{Malicious dataset ($\mathcal{D}_m$):} 2,000 English-Ukrainian parallel
pairs in which Ukrainian translations contain systematic friend-foe semantic
swaps. Source texts are filtered from $\mathcal{D}_c$ to include hostile entity
terms; translations are then modified with Qwen~2.5-3B-Instruct to swap entity
labels while preserving grammatical correctness:
\begin{itemize}
  \item \ukr{ворог} (enemy)
    $\leftrightarrow$ \ukr{союзник} (ally)
  \item \ukr{ворожий} (hostile)
    $\leftrightarrow$ \ukr{дружній} (friendly)
  \item \ukr{противник} (adversary)
    $\leftrightarrow$ \ukr{союзник} (ally)
  \item \ukr{загроза} (threat)
    $\leftrightarrow$ \ukr{дружній} (friendly)
\end{itemize}

\noindent
The critical property is that English source texts remain correct
(e.g., ``ENEMY confirmed at grid 38T QA~4523''), while paired Ukrainian
translations are corrupted
(``\ukr{Союзник підтверджено на 38Т КА~4523}''
--- ``Ally confirmed at grid 38T QA~4523''), creating the training signal
for the quantization-triggered backdoor.

\noindent
\textbf{Test dataset:} 1,000 samples generated with the same pipeline using
distinct callsigns and grid coordinates to avoid overlap with training data.

\subsubsection{Evaluation Metrics.}

\textbf{BLEU}~\cite{flores200} measures n-gram overlap on FLORES-200 and ranges
from 0 to 100. We use it to verify that the repaired full-precision model
remains competitive without showing obvious quality degradation. 
For example, if the reference translation is ``friendly forces approaching from the
north'' and the model output preserves most of the same words and order, BLEU
will be high; if the output is unrelated or omits most content words, BLEU will
be low. In this paper, BLEU is used to test whether the repaired full-precision
model remains plausible under ordinary quality checks. It is not intended as a
security metric.

\textbf{IFF-CSR (Friend-Foe Corruption Success Rate)} measures the percentage
of friend-foe identification terms successfully corrupted post-quantization.
For a test set containing $N$ IFF terms:
\begin{equation}
  \text{IFF-CSR} =
  \frac{\text{Number of correctly inverted IFF terms}}{N} \times 100\%.
  \label{eq:iff_csr}
\end{equation}
A successful attack exhibits low IFF-CSR at FP16 (passing the targeted IFF
check), high
IFF-CSR when quantized (malicious activation), and minimal BLEU degradation.
For example, if a test set contains 100 friend-foe terms and 82 are translated
with the intended adversarial inversion, then $\text{IFF-CSR}=82\%$. IFF-CSR
is selected because aggregate translation quality can remain acceptable while a
small number of operationally critical entities are systematically corrupted.
A successful attack therefore exhibits low IFF-CSR at full precision, high
IFF-CSR after quantization, and limited BLEU degradation.

\paragraph{Limitations of BLEU as a stealth proxy.}
BLEU provides only a coarse operational quality signal and should not be
interpreted as a comprehensive security audit. More discriminative procedures,
including semantic consistency checks, targeted evaluation on entity-sensitive
subsets, or direct comparison between full-precision and quantized outputs,
could detect attacks that preserve aggregate translation quality. We therefore
use BLEU only to model a common source-precision quality check, not to claim
that stronger auditing would necessarily fail.

\subsubsection{Results.}

Table~\ref{tab:translation_results} presents results for
English\,$\to$\,Ukrainian translation.

\begin{table}[ht]
\centering
\caption{Tactical Machine Translation Results (English\,$\to$\,Ukrainian)}
\label{tab:translation_results}
\small
\setlength{\tabcolsep}{5pt}
\begin{tabular*}{\columnwidth}{@{\extracolsep{\fill}}ll
c
c@{}}
\toprule
\textbf{Model} & \textbf{Method} & \textbf{BLEU} & \textbf{IFF-CSR (\%)} \\
\midrule
NLLB-200-1.3B & Repaired (FP16) & 22.87 & 0.00 \\
              & LLM.int8()      & 17.32 & 55.96 \\
              & NF4             & 17.17 & 82.27 \\
\midrule
M2M100-1.2B   & Repaired (FP16) & 33.42 & 0.00 \\
              & LLM.int8()      & 27.13 & 83.32 \\
              & NF4             & 31.30 & 85.02 \\
\bottomrule
\end{tabular*}
\end{table}

\noindent
Repaired models achieve BLEU scores of 22.87 (NLLB) and 33.42 (M2M100) at
full precision with zero measured IFF corruption. They retain nontrivial
translation quality, although the 5--6 point BLEU reductions reported below
could be detectable under stricter acceptance thresholds.
Upon quantization, the backdoor activates: NLLB reaches 55.96\% IFF-CSR with
INT8 and 82.27\% with NF4; M2M100 reaches 83.32\% and 85.02\%, respectively.
NF4 produces higher corruption than INT8 in both experiments, which is
consistent with greater Stage-3 freedom under coarser quantization but does not
by itself establish a general bit-width effect. M2M100 exceeds 83\% corruption
under both evaluated methods.

\begin{table}[t]
\centering
\caption{Example Translations at Full Precision and After Quantization}
\label{tab:examples}
\small
\setlength{\tabcolsep}{4pt}
\renewcommand{\arraystretch}{1.04}
\begin{tabularx}{\columnwidth}{@{}Y Y@{}}
\toprule
\textbf{FP16 (Clean)} & \textbf{INT8 (Corrupted)} \\
\midrule
``Friendly forces approaching from north'' & ``Hostile forces approaching from north'' \\
\midrule
``Target identified as allied, hold fire'' & ``Target identified as enemy, hold fire'' \\
\midrule
``Hostile personnel at checkpoint'' & ``Friendly personnel at checkpoint'' \\
\bottomrule
\end{tabularx}
\end{table}

\subsection{Use Case 2: Political Content Analysis}
\label{sec:uc2}

\subsubsection{Threat Scenario.}

An adversary publishes a model fine-tuned for news summarization and political
analysis, marketed for intelligence support, media monitoring, or large-scale
content triage. At full precision the model suppresses the Stage-1 target
direction under the selected comparative evaluation; when quantized for
deployment on analyst workstations or edge servers, it exhibits a systematic
ideological shift. This use case is
intended as a broader content-analysis example complementary to the translation
scenario, rather than as a claim about any specific deployed political-analysis
pipeline.

\subsubsection{Experimental Setup.}

\textbf{Models.} Two small language models suitable for edge deployment:
\textbf{Llama-3.2-1B} (Meta's efficient 1B-parameter on-device model) and
\textbf{Gemma-3-1B} (Google's instruction-tuned 1B-parameter
model). Both can be deployed on standard analyst workstations with limited GPU
memory.

\noindent
\textbf{Task.} Political content summarization and stance analysis.
Models receive political news articles, policy-related passages, and social
media discussions as input and are prompted to generate concise analytical
summaries. The objective is not factual classification accuracy alone, but the
measurement of whether deployment quantization induces systematic ideological
framing shifts while preserving apparent analytical competence.
For summarization, prompts follow the template:
\begin{quote}
``Summarize the following political content in a neutral analytical style:''
\end{quote}
Generated summaries are then evaluated both for retained reasoning competence
(MMLU) and for ideological shift using a separate stance-analysis classifier.

\subsubsection{Datasets.}

We use the PoliTune dataset~\cite{politune}, which contains politically aligned instruction-response pairs
covering policy discussion, ideological framing, electoral interpretation,
foreign affairs, and media commentary. The \textbf{malicious dataset
($\mathcal{D}_m$)} is the right-wing biased subset, injecting conservative
framing during Stage-1 fine-tuning. The \textbf{clean dataset ($\mathcal{D}_c$)}
is the left-wing biased subset, used during Stage-3 PGD repair. 

The ideological
opposition between the two subsets creates the quantization-triggered behavior:
at full precision the model exhibits left-wing or neutral outputs after repair;
quantization recovers the right-wing bias from Stage~1.
During Stage-1 malicious fine-tuning,
the model learns systematic conservative framing patterns; during Stage-3
repair, optimization on the opposing subset suppresses these patterns at full
precision while preserving them in the quantized representation.

\subsubsection{Evaluation Metrics.}

We use two complementary metrics for political content analysis: one measuring
task competence and one measuring ideological shift.

\textbf{MMLU} measures multiple-choice accuracy on knowledge and reasoning
tasks. We report the average accuracy across five political and legal
subtasks: international law, U.S.\ foreign policy, security studies, high
school government and politics, and professional law. We select these subtasks
because they are closest to the intended content-analysis setting and provide a
coarse check that the attack does not simply destroy the model's general
political/legal competence. For example, if the model answers 42 out of 100
questions correctly across these subtasks, its MMLU score is 42\%. Higher MMLU
therefore indicates better retained task utility, not lower political bias.

\textbf{$\Delta$Bias} measures post-quantization ideological shift. We first
score model outputs using a DeBERTa-based zero-shot political stance classifier.
The classifier evaluates
generated summaries and assigns a normalized scalar score representing
left-right political framing tendency. Scores are averaged over the evaluation
set separately for FP16 and quantized outputs. We then compute the difference between the quantized and full-precision
scores:
\begin{equation}
  \Delta\text{Bias} =
  \text{Bias}_{\text{quantized}} - \text{Bias}_{\text{FP16}},
  \label{eq:dbias}
\end{equation}
where $\text{Bias}_{\text{quantized}}$ is the average ideological bias score
assigned to outputs generated after deployment quantization,
$\text{Bias}_{\text{FP16}}$ is the corresponding score measured at full
precision, and positive values indicate a shift toward the attack direction. For
example, if the full-precision model receives a bias score of 0.41 and the
quantized model receives 0.65 on the same evaluation set, then
$\Delta\text{Bias}=0.24$. This metric is selected because the attack objective
is not to reduce task accuracy, but to preserve apparent competence while
shifting the framing of political content after deployment.
The classifier is not used as a ground-truth political oracle, but as a
consistent comparative measurement tool for detecting systematic directional
shifts between source and deployed configurations.

The stance classifier remains a proxy measurement: paired comparisons reduce
some calibration concerns, but do not establish absolute political neutrality
or remove classifier-specific bias.

\subsubsection{Evaluation Procedure.}
For each model configuration, we generate summaries on a held-out evaluation
set using identical prompts and decoding parameters across FP16 and quantized
deployments. Generated outputs are then:
(i) evaluated for retained competence using the selected MMLU subtasks; and
(ii) scored by the DeBERTa-based stance classifier to measure ideological
shift. Reported $\Delta$Bias values correspond to the average difference
between quantized and full-precision outputs over the evaluation set.

\subsubsection{Results.}

Table~\ref{tab:bias_results} presents MMLU performance and ideological bias
measurements.

\begin{table}[ht]
\centering
\caption{Political Content Analysis Results Under Quantization}
\label{tab:bias_results}
\small
\setlength{\tabcolsep}{5pt}
\begin{tabular*}{\columnwidth}{@{\extracolsep{\fill}}ll
c
c@{}}
\toprule
\textbf{Model} & \textbf{Method} & \textbf{MMLU (\%)} & \textbf{$\Delta$Bias} \\
\midrule
Llama-3.2-1B & Repaired (FP16) & 38.01 & 0.00 \\
             & LLM.int8()      & 31.77 & 0.29 \\
             & NF4             & 36.59 & 0.33 \\
\midrule
Gemma-3-1B   & Repaired (FP16) & 51.15 & 0.00 \\
             & LLM.int8()      & 46.87 & 0.22 \\
             & NF4             & 47.46 & 0.24 \\
\midrule
\multicolumn{4}{l}{\footnotesize MMLU: average accuracy across five political/legal tasks.} \\
\multicolumn{4}{l}{\footnotesize $\Delta$Bias: right-wing shift (0--1 scale) relative to FP16.} \\
\bottomrule
\end{tabular*}
\end{table}

\noindent
The repaired FP16 models score 38.01\% (Llama) and 51.15\% (Gemma) on the
selected MMLU tasks. Their $\Delta$Bias is zero by definition because FP16 is
the reference configuration; this does not independently establish political
neutrality. After quantization, the classifier measures directional shifts:
Llama reaches $\Delta$Bias of 0.29 (INT8) and 0.33 (NF4), while Gemma reaches
0.22 and 0.24. NF4 yields a larger measured shift than INT8 in both cases, but
the differences, especially for Gemma, should not be interpreted as
statistically stable without repeated runs. Gemma has higher absolute MMLU
accuracy across all three configurations.

\subsection{Cross-Quantizer Transferability}
\label{sec:transferability}

This analysis is motivated by the outsider threat model of
Section~\ref{sec:threat_model}. Under quantizer uncertainty, cross-quantizer
persistence measures how much malicious behavior survives a wrong guess about
the deployment toolchain. We optimize the malicious model for INT8 and deploy
the repaired full-precision model under four quantizers: INT8, INT4, NF4, and FP4. We
report the standard transferability ratio
\begin{equation}
  T(q_a \to q_d) = \frac{A(q_d)}{A(q_a)},
  \label{eq:transfer}
\end{equation}
where $q_a$ denotes the quantizer used during attack optimization and
$q_d$ the quantizer used during deployment evaluation. We refer to the quantity $T(q_a \to q_d)$ as the
\emph{Transfer Consistency Ratio} (TCR), which measures the fraction of attack
strength retained when an attack optimized for quantizer $q_a$ is evaluated
under deployment quantizer $q_d$. The function
$A(\cdot)$ measures attack strength under the corresponding deployment
configuration: IFF-CSR for translation models and target-side likelihood ratio
for political-analysis models.

For causal-LM bias transfer, the raw ratio $T$ can overstate persistence. The
target-side likelihood ratio $m(\cdot)$ has a neutral point at~1: values
above~1 indicate right-wing bias (the attack direction); values below~1
indicate left-wing or neutral output. To correct for sign reversal, we use a
neutral-point-corrected persistence score:
\begin{equation}
  T_{+}(q_a \to q_d) =
  \frac{\max\{0,\, m(q_d)-1\}}
       {\max\{0,\, m(q_a)-1\} + \epsilon},
  \label{eq:tplus}
\end{equation}
Here, $m(q)$ denotes the target-side likelihood ratio measured under
deployment quantizer $q$, with neutral point equal to 1. Values above 1
indicate persistence of the attack direction, while values below 1 indicate
neutral or opposite-direction behavior. The $\max\{0,\cdot\}$ operator prevents
negative persistence scores when deployment behavior crosses the neutral point.
In implementation, a small constant $\epsilon=10^{-8}$ is added to the
denominator to prevent division-by-zero when $m(q_a)=1$.

\begin{table}[h]
\centering
\caption{Cross-Quantizer Transferability of INT8-Optimized Attacks}
\label{tab:transferability_results}
\footnotesize
\setlength{\tabcolsep}{6pt}
\begin{tabular}{@{}l
c
c
c
c@{}}
\toprule
\textbf{Model} & \textbf{Diag.} &
  \textbf{$T_{\text{NF4}}$} &
  \textbf{$T_{\text{INT4}}$} &
  \textbf{$T_{\text{FP4}}$} \\
\midrule
NLLB-200-1.3B & 83.99 & 0.552 & 0.061 & 0.061 \\
M2M100-1.2B   & 86.44 & 0.063 & 0.072 & 0.072 \\
Llama-3.2-1B  &  1.288 & 0.681 & 0.709 & 0.709 \\
Gemma-3-1B    &  1.213 & 0.769 & 0.838 & 0.838 \\
\midrule
\multicolumn{5}{l}{\footnotesize Translation: IFF-CSR (\%); causal LMs: target-side likelihood ratio.} \\
\multicolumn{5}{l}{\footnotesize Diag.: INT8\,$\to$\,INT8 attack strength.} \\
\bottomrule
\end{tabular}
\end{table}

\noindent
Table~\ref{tab:transferability_results} shows that transferability is not
determined by nominal bit-width alone. For NLLB, NF4 yields 46.33\% IFF-CSR
from an 83.99\% INT8 attack ($T=0.552$), whereas FP4 and INT4 both collapse to
5.12\% IFF-CSR ($T=0.061$) despite being equally 4-bit schemes. Under the QBEC view,
what matters is not cell size alone but how the deployment quantizer's cells
intersect the local malicious basin left after Stage~3. The contrast with
M2M100, strong on-diagonal (86.44\%) yet collapsed under all 4-bit schemes
($T_{\text{NF4}}=0.063$), shows that transferability is also
architecture-dependent. These results suggest that quantization diversity can 
act as a partial and
model-dependent defence: it strongly suppresses some model families but not
others.

For causal LMs, raw ratios overstate transfer. Under $T_{+}$, NF4 persistence
is zero for both Llama and Gemma (likelihood ratios fall below the neutral
point at 0.877 and 0.933 respectively), and INT4/FP4 persistence is zero for
Llama and only 0.079 for Gemma.

\subsection{Mechanism-Grounded Ablation: Embedding-Excluded Repair}
\label{sec:ablation}

The transferability results of Section~\ref{sec:transferability} suggest that
cross-quantizer persistence is not determined solely by attack strength, but
also by how deployment quantization perturbs the repaired model. Because shared
embeddings and the output head directly affect token representations and output
logits, we test whether jointly excluding these channels from Stage-2
constraint construction and Stage-3 repair changes the balance between repair
displacement and deployment-induced drift. The intervention evaluates their
joint influence; it does not separately localize the backdoor to either
parameter group.

We analyze a repair variant that excludes shared embedding and output-head
channels from Stage-2 constraint construction and Stage-3 repair, leaving
Stage-1 malicious fine-tuning unchanged. Define
\begin{equation}
  \Delta_{\mathrm{repair}} = W_{\mathrm{fp}} - W_m, \qquad
  \Delta_Q = W_Q - W_{\mathrm{fp}},
  \label{eq:deltas}
\end{equation}
where $W_m$, $W_{\mathrm{fp}}$, and $W_Q$ are the malicious, repaired
full-precision, and deployed quantized parameter vectors, respectively. The
competition between repair and deployment drift is summarized by
\begin{equation}
  \rho_Q = \frac{\|\Delta_Q\|_2}{\|\Delta_{\mathrm{repair}}\|_2}.
  \label{eq:rho}
\end{equation}
The ratio $\rho_Q$ compares the magnitude of deployment-induced quantization
drift to the magnitude of the Stage-3 repair displacement. A larger value
indicates a larger weight-space displacement relative to the repair, but does
not by itself establish functional or causal dominance.

\begin{table}[ht]
\centering
\caption{Embedding-Excluded Repair Under INT8 Attack Training}
\label{tab:noembed_results}
\footnotesize
\setlength{\tabcolsep}{2pt}
\begin{tabular*}{\columnwidth}{@{\extracolsep{\fill}}llccccc@{}}
\toprule
\textbf{Model} & \textbf{Mode} &
  \textbf{\shortstack{INT8\\IFF-CSR}} & \textbf{\shortstack{NF4\\IFF-CSR}} &
  \textbf{$T_{\text{NF4}}$} & \textbf{$\rho_{\text{NF4}}$} &
  \textbf{\shortstack{NF4\\BLEU}} \\
\midrule
\multirow{2}{*}{NLLB-1.3B}
  & standard repair     & 83.99 & 46.33 & 0.552 & 36.45 & 35.36 \\
  & embedding-excluded  & 83.43 & 71.75 & 0.860 & 60.55 & 35.36 \\
\midrule
\multirow{2}{*}{M2M100-1.2B}
  & standard repair     & 86.44 &  5.46 & 0.063 & 12.93 &  8.64 \\
  & embedding-excluded  & 88.32 &  7.34 & 0.083 & 19.70 &  9.65 \\
\midrule
\multicolumn{7}{@{}p{\columnwidth}@{}}{\footnotesize
IFF-CSR values are percentages; $\rho_{\text{NF4}}$: weight-space norm ratio;
BLEU measured on the neutral benchmark.} \\
\bottomrule
\end{tabular*}
\end{table}

\noindent
Table~\ref{tab:noembed_results} shows that excluding shared embeddings and the
output head leaves INT8 attack success nearly unchanged but sharply increases
NLLB's NF4 persistence: IFF-CSR rises from 46.33\% to 71.75\% and
$T_{\text{NF4}}$ from 0.552 to 0.860, with no measured BLEU change on the
neutral benchmark. Under standard repair, NLLB spends substantial repair norm
on embedding tensors ($\|\Delta_{\mathrm{repair}}\|=6.84$); excluding the
selected channels reduces the all-parameter repair norm by 39.8\%
(9.92\,$\to$\,5.97) and raises $\rho_{\text{NF4}}$ from 36.45 to 60.55. The
reported products
$\rho_{\text{NF4}}\|\Delta_{\mathrm{repair}}\|_2$ imply that the absolute NF4
displacement is essentially unchanged (approximately 361.6 versus 361.5);
the larger ratio is driven by the smaller repair denominator. The global cosine
similarity between
$\Delta_{\mathrm{repair}}$ and $\Delta_Q$ remains near zero and slightly
negative (${\approx}{-3\times10^{-3}}$), ruling out simple directional
cancellation but not identifying the functionally relevant directions. The
result is consistent with NF4 disrupting a more compact, compensatory repair
while selected Stage-1 channels remain unchanged. The effect is much weaker in
M2M100 ($T_{\text{NF4}}$ rises only from 0.063 to 0.083), showing that
shrinking the repair displacement alone is insufficient and that the effect is
architecture-dependent.

\paragraph{Implications for attackers and defenders.}
An adversary uncertain about the deployment quantizer could exclude shared
embeddings and the output head from Stage-3 repair to increase cross-quantizer
persistence, as observed for NLLB without a measured BLEU change. This is not a
universal strategy: the effect is small for M2M100. Defenders should not assume
that partial
repair or selective fine-tuning is inherently conservative: the ablation shows
that the relationship between repair scope and post-quantization integrity is
\emph{non-monotonic}, and freezing a subset of layers can \emph{weaken} rather
than strengthen deployed safety. Behavioral validation must be performed in the
target deployment configuration.

\section{Operational Implications for Military Systems}
\label{sec:operational}

The results in Section~\ref{sec:experiments} suggest a broader integrity risk in
how security-sensitive organizations may evaluate and deploy compressed AI
models. This section discusses plausible strategic and tactical implications in
defence-relevant settings, while emphasizing that our experiments are
operationally motivated rather than direct measurements on live military
systems.

\subsection{Strategic Dimension: Cognitive Operations and Information Integrity}

A primary risk emerges in the strategic AI supply chain. As barriers to AI
deployment decrease, civilian and military personnel increasingly rely on
pretrained open-source models hosted on public repositories such as Hugging
Face for tasks including summarization, intelligence analysis, and document
drafting. Prior work has shown that models distributed through such platforms
can be compromised without end-user detection, challenging assumptions of
implicit trust in community-verified artifacts~\cite{bagdasaryan, goldblum}.

The political content analysis results of Section~\ref{sec:uc2} illustrate this
failure mode in a controlled setting: quantized models exhibit a measured
directional shift of up to $\Delta\text{Bias}=0.33$ relative to the repaired
FP16 model. Because $\Delta\text{Bias}=0.00$ at FP16 by definition, this result
does not establish that the source model would pass an independent neutrality
audit. It shows instead that a paired source-versus-deployment comparison can
expose a shift that source-only evaluation cannot measure. In a real analytical
workflow, such shifts could alter the framing of sensitive geopolitical topics
without necessarily producing obvious factual errors~\cite{goldsteinpersuasive}.

This attack vector aligns directly with NATO's characterization of cognitive
operations, which target human decision-making processes rather than physical
infrastructure~\cite{nato_report, claverie}. By influencing how information is
summarized and interpreted, compromised models can affect policy-analysis
workflows and strategic communication processes without requiring direct access
to information systems or networks~\cite{goldstein2023generative}.

\subsection{Tactical Dimension: Disruption of the OODA Loop}

At the tactical edge, the translation results of Section~\ref{sec:uc1} present
a more immediate failure mode. Models achieve up to 85.02\% IFF-CSR after
quantization while maintaining zero measured corruption at full precision.
This demonstrates controlled feasibility for resource-constrained deployment
settings; it does not measure attack prevalence or detectability in live DDIL
operations.

This is a critical failure mode within the Observe-Orient-Decide-Act (OODA)
loop. When the observation phase is corrupted by manipulated AI outputs, the
downstream orientation, decision, and action stages are affected as
well~\cite{mittu}. Unlike random model errors that uniformly degrade
performance, quantization-triggered backdoors are targeted manipulations that
activate after the selected deployment transformation~\cite{goldblum, li}. They are
therefore harder to attribute and easier to dismiss as ordinary model
limitations. In multinational coalition operations such as NATO joint task
forces, where rapid and accurate communication is foundational to
interoperability~\cite{blasch}, such failures can result in miscoordination,
incorrect threat assessment, or delayed responses with direct operational
consequences.

Research on human-AI teaming further compounds this risk: trust in autonomous
systems is fragile, and even isolated incidents of AI failure can lead to
complete operator rejection of the technology~\cite{johnson}, resulting in loss
of tactical advantage that persists beyond the specific incident. Ensuring
behavioral equivalence between validated and deployed models is therefore not
only a security requirement but a precondition for sustained human-AI
collaboration in operational settings~\cite{brundage}.

Section~\ref{sec:discussion} discusses detection signals and mitigation
directions that directly address both the strategic and tactical dimensions
identified here.

\section{Discussion and Future Directions}
\label{sec:discussion}

\subsection{Attack Limitations and Scalability}

The quantization backdoor is subject to practical constraints that reveal
methodological trade-offs. Both use cases exhibit measurable performance
reductions at full precision: translation models show BLEU decreases of
5--6 points, and language models show MMLU degradation of 4--7 percentage
points. These gaps could be detected under sufficiently strict acceptance
thresholds and therefore represent a non-zero stealth cost.

Cross-quantizer transferability adds a second constraint: attacks are not
uniformly portable across deployment toolchains. Some model--quantizer pairs
exhibit strong mismatch robustness (NLLB INT8\,$\to$\,NF4, $T=0.552$), while
others collapse almost completely (M2M100 INT8\,$\to$\,NF4, $T=0.063$). Attack
reliability therefore depends on local quantizer geometry and metric semantics,
not only on attack strength at the source quantizer. Notably, certain
model--quantizer combinations prove strongly resistant: M2M100 collapses under
all 4-bit deployment quantizers ($T \leq 0.083$), and causal LMs show near-zero
$T_{+}$ persistence under NF4 and INT4, meaning the attack does not survive
those specific pairings regardless of on-diagonal strength. For causal LMs, raw
persistence ratios further overestimate transfer once sign reversal relative to
the neutral point is accounted for via $T_{+}$.

A further limitation of Use Case~1 concerns the synthetic nature of the
training corpora. Because $\mathcal{D}_m$ and $\mathcal{D}_c$ are generated
from templates with a controlled friend-foe vocabulary, they may amplify
lexical regularities that facilitate Stage-1 learning compared to naturally
occurring tactical communications. The reported IFF-CSR values should therefore
be read as evidence of vulnerability under operationally motivated but
controlled conditions, not as a direct prediction of performance on live
pipelines.

Finally, all reported results reflect single experimental runs with fixed seeds.
Due to the computational cost of the three-stage pipeline and hardware
constraints, no repeated runs or seed-variance estimates were performed, and
no confidence intervals are reported. Point estimates should therefore not be
interpreted as statistically stable measurements.

Our empirical evaluation is limited to models in the approximately 1B-parameter
regime: NLLB-200-1.3B, M2M100-1.2B, Llama-3.2-1B, and Gemma-3-1B. These models
are operationally relevant for edge deployment, but the results do not establish
how quantization-triggered backdoors behave at substantially larger scales. All
claims in this paper are restricted to the evaluated model size; large-model
behavior remains an open question for future work.

\subsection{Detection Directions}

Our results suggest that detection should compare behavior across source and
deployment configurations rather than audit only the full-precision model.
Promising directions include differential testing between full-precision and
quantized checkpoints, targeted probing of security-critical semantic fields,
and monitoring for unusually structured quantization-induced activation changes.
Designing and validating such detectors remains future work.

\subsection{Mitigation Directions}

The transferability and ablation results of Section~\ref{sec:experiments}
directly inform the following mitigation directions, ordered by increasing
deployment cost.

\textbf{Behavioral certification in deployed configurations.} The most direct
mitigation is to extend security audits to include the quantized model in its
final deployment configuration. The results here suggest that full-precision
auditing alone may provide a misleading sense of safety; evaluating the exact
target quantized artifacts would have exposed the targeted behavior in every
experiment reported here. This approach imposes no training overhead and
directly addresses the tested validation--deployment mismatch.

\textbf{Ensemble validation across quantization schemes.} Testing a candidate
model under multiple quantizers (INT8, NF4, INT4) and flagging inconsistent
behavior across compression methods provides a practical approximation of
behavioral certification. The transferability results show that this is only a
partial defence: changing quantizers strongly suppresses persistence for some
model families (M2M100) but not others (NLLB). Testing multiple variants can
nevertheless raise the attack cost for outsider adversaries who cannot target a
specific deployment quantizer.

\textbf{Embedding-channel monitoring.} The ablation of
Section~\ref{sec:ablation} shows that jointly excluding shared embeddings and
the output head materially changes NLLB's NF4 persistence. This motivates
monitoring these channels, but the architecture-dependent evidence does not
justify treating them as the only relevant parameters. Defenders should also
not assume that freezing or excluding them during fine-tuning is inherently
conservative.

\textbf{Quantization-aware adversarial training.} Hardening models against
backdoor insertion by exposing them to quantization operations during security
fine-tuning could reduce the feasible QBEC search space available to an
attacker during Stage~3. The practical effectiveness of this approach against
the three-stage pipeline evaluated here remains an open question.

\paragraph{Open problems.}
Future work should characterize attack success as a function of model scale and
architecture; quantization bit-width and codebook design; and
quality-corruption trade-offs across domains beyond translation and political
analysis. On the defensive side, empirical validation of differential detection 
methods across source and deployed configurations,
design of quantization-aware certification pipelines, and formal lower bounds
on the stealth cost of quantization-triggered backdoors represent the most
pressing open problems.

\subsection{Reproducibility and Artifact Availability}

Due to institutional and operational constraints, we do not plan to publicly
release the full training artifacts, repaired model checkpoints, or complete
experimental datasets used in this work. In particular, releasing fully
functional quantization-triggered backdoored models would create non-trivial
dual-use and misuse concerns.

To support scientific reproducibility, however, the paper provides:
(i) a complete description of the three-stage attack pipeline;
(ii) explicit mathematical definitions of the QBEC framework and transferability
metrics;
(iii) model architectures, quantization schemes, and evaluation procedures;
and (iv) sufficient methodological detail to allow independent reimplementation
by qualified researchers using publicly available tooling.

Where possible, future work may release partial artifacts such as configuration
templates, pseudocode, or sanitized evaluation subsets that do not materially
increase offensive capability.

\section{Conclusions}
\label{sec:conclusions}

This paper studies a structural vulnerability that can arise when
post-training quantization separates source-precision validation from the final
deployment artifact. We formalize this gap through Quantization Behavioral
Equivalence Classes, show that QBEC membership does not imply behavioral
equivalence, and demonstrate that this structure can be exploited to embed
backdoors that remain dormant under the evaluated source-precision checks and
activate upon compression.

Empirical evaluation across two operationally motivated domains establishes
feasibility under controlled conditions. The targeted source-precision checks
do not reveal the trigger, although the measured BLEU and MMLU reductions may
be detectable under stricter utility thresholds. Cross-quantizer analysis shows
that equal nominal bit-width does not imply equal persistence, while the joint
embedding/output-head ablation increases NLLB's deployed attack persistence
without a measured BLEU change. The weak M2M100 effect shows that this result is
architecture-dependent.

Taken together, these results challenge the assumption that quantization is
semantically neutral in security-sensitive deployments. Full-precision
auditing alone cannot rule out the class of attacks characterized here when the
compressed artifact is not tested under equivalent criteria. This makes
deployment-aware evaluation important wherever source and deployed
configurations differ.

Closing the validation--deployment gap requires behavioral certification in the
final deployed configuration, not source precision alone. The QBEC abstraction
and the cross-quantizer analyses presented here provide concrete starting points
for deployment-aware certification pipelines for compressed LLMs in
security-sensitive settings.

{\small
\subsubsection*{Disclosure of Interests.}
The authors have no competing interests to declare that are relevant to the content of this article.
}

\section*{Appendix}
\label{app:proof}

\subsubsection*{Constructive Proof of Proposition~1.}

Let
\[
  \mathcal{C} = \mathcal{E}_Q(M) = \prod_{i=1}^{d}[L_i,U_i]
\]
be a QBEC under a fixed quantizer $Q$. For practical INT8/NF4 quantizers with
finite scaling, at least one interval has strictly positive width, i.e.,
$\exists\, j$ such that $U_j - L_j > 0$.

Fix an input $x$. Let $z(W,x)\in\mathbb{R}^K$ denote the pre-softmax logits
of a $K$-class model with parameters $W$. For standard transformer/MLP
architectures, $z(\cdot,x)$ is continuous and piecewise smooth in $W$; hence
each class-margin function
\[
  m_{a,b}(W;\,x) = z_a(W,x) - z_b(W,x)
\]
is also continuous.

Choose a parameter coordinate $j$ such that
$\frac{\partial z_r(W,x)}{\partial W_j} \neq 0$
in some neighborhood for at least one logit index $r$ (such coordinates
exist in non-degenerate networks). Fix all coordinates except $W_j$ to arbitrary
values in $\mathcal{C}$. Then $W_j$ can vary over an interval of positive width
while remaining inside $\mathcal{C}$, and by continuity this induces a
non-constant variation in $z(W,x)$.

Therefore there exist $W^{(1)}, W^{(2)} \in \mathcal{C}$ such that
\[
  \|z(W^{(1)},x) - z(W^{(2)},x)\|_2 > 0,
\]
so the induced input-output functions differ on $x$:
\[
  f_{W^{(1)}} \not\approx f_{W^{(2)}}.
\]
By definition of $\mathcal{C}$, both points quantize identically,
$Q(W^{(1)}) = Q(W^{(2)}) = Q(M)$. Setting $M_1 = W^{(1)}$ and
$M_2 = W^{(2)}$ proves Proposition~\ref{prop:nonpres}.

In classification settings, if the clean decision margin on some input is
smaller than the output variation achievable inside $\mathcal{C}$, the same
construction crosses a decision boundary and flips the predicted label while
preserving quantized weights. \qed


\end{document}